\documentclass[11pt]{article}
\usepackage[margin=1.15in]{geometry}
\usepackage{amsmath,amssymb,amsthm,mathtools}
\usepackage[colorlinks=true,linkcolor=blue,citecolor=blue,urlcolor=blue]{hyperref}
\usepackage[most]{tcolorbox}
\tcbuselibrary{listings,breakable}
\usepackage{dirtytalk}

\usepackage{tikz}
\usetikzlibrary{arrows.meta,positioning}

\newtcblisting{promptbox}{
  breakable,
  listing only,
  colback=gray!3,
  colframe=black!60,
  boxrule=0.5pt,
  arc=1mm,
  left=6pt,
  right=6pt,
  top=6pt,
  bottom=6pt,
  title={Distilled prompt, GPT-5.5 Pro, June 2026},
  fonttitle=\bfseries\small,
  listing options={
    basicstyle=\ttfamily\tiny,
    breaklines=true,
    columns=fullflexible,
    keepspaces=true,
    showstringspaces=false,
    upquote=true
  }
}

\newtcblisting{removedhintbox}{
breakable,
listing only,
colback=red!2,
colframe=red!70!black,
boxrule=0.6pt,
arc=1mm,
left=6pt,
right=6pt,
top=6pt,
bottom=6pt,
title={Additional removed hints, GPT-5.5 Pro, June 2026},
fonttitle=\bfseries\small,
listing options={
basicstyle=\ttfamily\tiny,
breaklines=true,
columns=fullflexible,
keepspaces=true,
showstringspaces=false,
upquote=true
}
}

\newcommand{\X}{\mathcal X}
\newcommand{\F}{\mathcal F}
\newcommand{\C}{\mathcal C}
\newcommand{\E}{\mathbb E}
\newcommand{\Prb}{\mathbb P}
\newcommand{\1}{\mathbf 1}
\newcommand{\err}{\operatorname{err}}
\newcommand{\Maj}{\operatorname{Maj}}
\newcommand{\vc}{\operatorname{VC}}

\theoremstyle{plain}
\newtheorem{theorem}{Theorem}[section]
\newtheorem{proposition}[theorem]{Proposition}
\newtheorem{lemma}[theorem]{Lemma}

\theoremstyle{remark}

\title{Majority-of-Three is Optimal}

\author{
Divit Rawal\thanks{Department of Statistics, University of California, Berkeley. Email: \{\href{mailto:divit.rawal@berkeley.edu}{divit.rawal}, \href{mailto:zhivotovskiy@berkeley.edu}{zhivotovskiy}\}@berkeley.edu.}
\quad
Nikita Zhivotovskiy\footnotemark[2]
}

\date{June 11, 2026}

\begin{document}

\maketitle

\begin{abstract}
We give a short proof that the majority vote of three independent consistent classifiers is an optimal learner in the realizable PAC setting. This proves optimality for the simplest voting scheme, while simplifying both the algorithmic structure and the probabilistic analysis of previous voting learners, including the algorithm of Hanneke \cite{Hanneke2016} and the analysis of bagging by Larsen \cite{Larsen2023Bagging}. 
\end{abstract}

\section{Introduction}

Determining the optimal sample complexity of PAC learning in the realizable setting has been one of the central problems in statistical machine learning theory since the introduction of VC theory and the PAC model \cite{VapnikChervonenkis1971,Valiant1984,BlumerEhrenfeuchtHausslerWarmuth1989}.  The difficulty is twofold.  First, some 
sample consistent classifiers are known to be suboptimal, so one must identify a suitable algorithmic construction.  Second, even for natural candidate algorithms, the required probabilistic analysis has proved delicate.

A major milestone was Hanneke's optimal PAC learner \cite{Hanneke2016}, building on the construction of Simon \cite{Simon2015}.  This resolved the question up to universal constants, but the resulting algorithm and its analysis are quite involved.  Subsequent work has therefore sought simpler optimal learners or simpler proofs, including bagging \cite{Larsen2023Bagging} and constructions based on the one-inclusion graph algorithm of Haussler, Littlestone, and Warmuth \cite{HausslerLittlestoneWarmuth1994,AdenAliCherapanamjeriShettyZhivotovskiy2023OIGNotAlwaysOptimal, AdenAliCherapanamjeriShettyZhivotovskiy2023OptimalPACNoUC}. 

A particularly elementary candidate in this direction is the majority vote of three independently trained consistent classifiers: split the sample into three equal blocks, train one consistent classifier on each block, and classify by majority vote.  This majority-of-three rule was studied in \cite{AdenAliHogsgaardLarsenZhivotovskiy2024}, where the optimal in expectation bound was proved, together with a near optimal PAC bound (up to a doubly logarithmic factor $\log\log(\min\{\frac{n}{d}, \frac{1}{\delta}\})$, where $n$ is the sample size, $d$ is the VC dimension, and $\delta$ is the confidence parameter). The authors of \cite{AdenAliHogsgaardLarsenZhivotovskiy2024} conjectured that majority-of-three achieves the optimal PAC bound; here we provide a proof of that claim.

\section{Main result}

Let $\F\subseteq\{0,1\}^{\X}$ have VC dimension $d\ge 1$, let $P$ be a distribution on $\X$, and let $f^\star\in\F$ be the unknown target.  For a training sample $S=(X_1,\ldots,X_m)\sim P^m$, write $(S,f^\star(S))$ for the labeled sample and let $\widehat f_S$ be the output of a fixed deterministic consistent selector; thus $\widehat f_S\in\F$ and $\widehat f_S(X_i)=f^\star(X_i)$ for every sample point $X_i$.  For a classifier $f$, write
\[
\err_P(f)=\Prb_{X\sim P}\left[f(X)\ne f^\star(X)\right],
\qquad
\Maj(f_1,f_2,f_3)(x)=\1\{f_1(x)+f_2(x)+f_3(x)>3/2\}.
\]
For simplicity, we assume in what follows that $n = 3m$ for some positive integer $m$. Given a sample of size $n$, split it into three independent blocks $S_1,S_2,S_3$ of equal sizes $m = n/3$, and put $\widehat f_i=\widehat f_{S_i}$.  Importantly, by construction $\widehat f_1, \widehat f_2, \widehat f_3$ are independent. In what follows, we ignore routine measurability issues. 

\begin{theorem}\label{thm:main}
Assume that $n = 3m$ for some positive integer $m$. For every $\delta\in(0,1)$, every distribution $P$, and every $f^\star \in \F$,
\[
\Prb\left(\err_P\left(\Maj(\widehat f_1,\widehat f_2,\widehat f_3)\right)
\le 10^7\frac{d+\log(1/\delta)}{n}\right)\ge 1-\delta .
\]
This matches the classic lower bound \cite{ehrenfeucht1989general} up to a universal constant. 
\end{theorem}

The proof is based on the following VC range space result.
\begin{proposition}[Moments of two block overlaps]
\label{prop:moment}
Let $P$ be a distribution on $\X$ and let $\C$ be a VC class of subsets of $\X$ with $\vc(\C)=d\ge1$.  Let $A$ be any deterministic selector satisfying $A(S)\in\C$ and $A(S)\cap S=\emptyset$ for every sample $S\in\X^m$.  If $S,T\sim P^m$ are independent, then for every integer $q\ge 1$,
\[
\left\|P(A(S)\cap A(T))\right\|_{L_q}
\le 62400\frac{d+q}{m}.
\]
\end{proposition}
This result is analogous to the overlap estimate of Simon \cite{Simon2015}, which is a central tool in subsequent analyses
\cite{Hanneke2016,Larsen2023Bagging,AdenAliHogsgaardLarsenZhivotovskiy2024}. The difference is that Simon's bound applies to nested samples, whereas here
$S$ and $T$ are independent. The nested sample structure is responsible for the additional doubly logarithmic factors in that approach \cite{AdenAliHogsgaardLarsenZhivotovskiy2024}.

\begin{proof}[Proof of Theorem~\ref{thm:main} assuming Proposition~\ref{prop:moment}]
Define the collection of error sets
\[
\C=\left\{\{x\in\X:f(x)\ne f^\star(x)\}: f\in\F \right\}.
\]
Observe that $\vc(\C)=\vc(\F)=d$. Thus,  if $S\sim P^m$ and
$
A(S)=\{x:\widehat f_S(x)\ne f^\star(x)\},
$
then $ A(S)\in\C$,
and $A(S)\cap S=\emptyset$. Denote  $E_i=\{x:\widehat f_i(x)\ne f^\star(x)\}$ for $i= 1,2,3$ and note that $E_i \in \C$. Just as in \cite{Simon2015}, if the majority vote misclassifies $x$, at least two of the three classifiers are incorrect at $x$. Therefore, it holds that
\[
\err_P\left(\Maj(\widehat f_1,\widehat f_2,\widehat f_3)\right)
\le  P(E_1\cap E_2) + P(E_2\cap E_3) + P(E_1\cap E_3) .
\]
By the triangle inequality and Proposition~\ref{prop:moment} we have
\begin{align*}
\left\|\err_P\left(\Maj(\widehat f_1,\widehat f_2,\widehat f_3)\right)\right\|_{L_q} &\le \left\|P(E_1\cap E_2)\right\|_{L_q}  + \left\|P(E_2\cap E_3)\right\|_{L_q}  + \left\|P(E_1\cap E_3)\right\|_{L_q} 
\\
&\le 561600\,\frac{d+q}{n}.
\end{align*}
Choose $q=\lceil \log(1/\delta)\rceil$.  Then $q\le \log(e/\delta)$, and Markov's inequality gives 
\[
\Prb\left(\err_P\left(\Maj(\widehat f_1,\widehat f_2,\widehat f_3)\right)>e\|\err_P\left(\Maj(\widehat f_1,\widehat f_2,\widehat f_3)\right)\|_{L_q}\right)\le \exp(-q)\le\delta.
\]
Simplifying the constants, this proves Theorem~\ref{thm:main}.
\end{proof}

\section{Proof of Proposition \ref{prop:moment}}

We start with the moment version of the double counting argument used in the analysis of \cite{AdenAliHogsgaardLarsenZhivotovskiy2024}. For $S \sim P^m$ and $x_1, \ldots, x_k \in \X$ define
 \begin{equation}
\label{eq:pq}
p_k(x_{1:k})=\Prb_S(x_1,\ldots,x_k\in A(S)).
\end{equation}
If $S,T\sim P^m$ and $X=(X_1,\ldots,X_q)\sim P^q$ are independent, then, by Fubini and independence, it holds that
\begin{equation}
\label{eq:momentfubini}
\E_{S,T}P(A(S)\cap A(T))^q
=\E_X\Prb_S(X_1,\ldots,X_q\in A(S))
        \Prb_T(X_1,\ldots,X_q\in A(T))
=\E_Xp_q(X)^2.
\end{equation}
Thus, by \eqref{eq:momentfubini},  the problem of upper bounding the moment reduces to understanding the behavior of the random variable $p_q(X)$. In what follows, we need the following explicit form of the VC bound (implied by the realizable PAC bound for consistent classifiers \cite[Theorem~2]{BlumerEhrenfeuchtHausslerWarmuth1989}). For every probability measure $Q$, every integer $N\ge0$, and every $\varepsilon\in(0,1]$,
\begin{equation}
\Prb_{Y_1,\ldots,Y_N\sim Q}
\left(\exists R\in\C:R\cap\{Y_1,\ldots,Y_N\}=\emptyset,\ Q(R)>\varepsilon\right)
\le
2\left(\frac{13}{\varepsilon}\right)^{2d}\exp\left(-\varepsilon N/4\right).
\label{eq:VC-net-explicit}
\end{equation}
To analyze $p_q(X)$, we use a recursive chain rule decomposition.  At each
step, after fixing a prefix $z=(z_1,\ldots,z_r)$, the relevant quantity is the
conditional inclusion probability
\[
\alpha_z(x)
=
\Prb_S\left(x\in A(S)\mid z_1,\ldots,z_r\in A(S)\right).
\]
Thus, we need a uniform tail bound
on these conditional probabilities, with explicit dependence on the rarity of
the conditioning event
$
p_z=\Prb_S\left(z_1,\ldots,z_r\in A(S)\right).
$
The next lemma provides exactly this form of conditional control.  
\begin{lemma}[Tail bound for conditional inclusion probabilities]
\label{lem:conditioned}
Let $S\sim P^m$ and let $A(S)\in\C$ satisfy $A(S)\cap S=\emptyset$, where $\vc(\C)=d\ge1$.  Fix $z=(z_1,\ldots,z_r)$ and define
$
\mathcal E_z=\{z_1,\ldots,z_r\in A(S)\}$, and $
p_z=\Prb_S(\mathcal E_z)$.
If $p_z>0$, set
\[
\alpha_z(x)=\Prb_S(x\in A(S)\mid \mathcal E_z) = \frac{\Prb_S(x, z_1, \ldots, z_r\in A(S))}{\Prb_S(z_1, \ldots, z_r\in A(S))}.
\]
Then, for every $a\in(0,1]$, the following tail bound for the conditional inclusion probability holds for $X \sim P$,
\[
\Prb\left(\alpha_z(X)>a\right)
 \le \frac{160}{am}\left(d\log\frac ea+\log\frac e{p_z}\right).
\]
\end{lemma}

\begin{proof}
Let
$
G=\{x:\alpha_z(x)>a\},$ and $\beta=P(G) = \Prb\left(\alpha_z(X)>a\right).
$
There is nothing to prove if $\beta=0$, so assume $\beta>0$, and let $P_G$ denote $P$ conditioned on $G$. By the definition of $G$, conditioning on $\mathcal E_z$, the random set $A(S)$ has average $P_G$-mass larger than $a$. Indeed, by Fubini,
\begin{align*}
\E\left[P_G(A(S))\mid \mathcal E_z\right]
&=
\E\left[\frac1\beta\int_G \mathbf 1\{x\in A(S)\}\,\mathrm dP(x)\,\middle|\,\mathcal E_z\right]
\\
&=
\frac1\beta\int_G \Prb_S(x\in A(S)\mid \mathcal E_z)\,\mathrm dP(x)
=
\frac1\beta\int_G \alpha_z(x)\,\mathrm dP(x)
>a.
\end{align*}
Since $0\le P_G(A(S)) \le 1$, the inequality
\[
a<\E[P_G(A(S))\mid\mathcal E_z]\le \frac a2+\Prb\left(P_G(A(S))>a/2\mid\mathcal E_z\right),
\]
multiplied by $p_z=\Prb(\mathcal E_z)$, implies
\begin{equation}
\label{eq:pgas}
\Prb\left(P_G(A(S))>a/2\right)
\ge
\frac{ap_z}{2}.
\end{equation}

We now bound the same probability using the VC bound \eqref{eq:VC-net-explicit}. By the definition of $A(S)$ and \eqref{eq:pgas}, it holds that
\begin{equation}
\label{eq:compareprobs}
\{P_G(A(S))>a/2\}
\subseteq
\{\exists R\in\C:\ R\cap S=\emptyset,\ P_G(R)>a/2\}.
\end{equation}
Let $N=\sum_{i=1}^m \mathbf 1\{X_i\in G\}$, where $S = (X_1, \ldots, X_m)$. Conditional on $N$, the points of $S\cap G$ are $N$ independent samples from $P_G$. Applying \eqref{eq:VC-net-explicit} to the restricted range space
$
\{R\cap G:R\in\C\},
$
which still has VC dimension at most $d$, with $Q=P_G$ and $\varepsilon=a/2$, gives
\[
\Prb\left(\exists R\in\C:R\cap S=\emptyset,\ P_G(R)>a/2\mid N\right)
\le
2\left(\frac{26}{a}\right)^{2d}\exp(-aN/8).
\]
Since $N\sim\operatorname{Bin}(m,\beta)$, we obtain
\[
\E\exp(-aN/8)
=
\left(1-\beta+\beta \exp(-a/8)\right)^m
=
\left(1-\beta(1-\exp(-a/8))\right)^m
\le
\exp(-a\beta m/16),
\]
where we used $1-\exp(-t)\ge t/2$ for $0\le t\le 1/8$. Therefore,
\begin{equation}
\label{eq:uniformtail}
\Prb\left(\exists R\in\C:\ R\cap S=\emptyset,\ P_G(R)>a/2\right)
\le
2\left(\frac{26}{a}\right)^{2d}\exp\left(-a\beta m/16\right).
\end{equation}
Combining the lower bound \eqref{eq:pgas} with the upper bound \eqref{eq:uniformtail}, using the inclusion \eqref{eq:compareprobs}, we get
\[
\frac{ap_z}{2}
\le
2\left(\frac{26}{a}\right)^{2d}\exp\left(-a\beta m/16\right).
\]
Taking logarithms and simplifying yields
\[
\frac{a\beta m}{16}
\le
\log\frac{4}{ap_z}+2d\log\frac{26}{a} \le 10d\log\frac ea+\log\frac e{p_z} \le 10\left(d\log\frac ea+\log\frac e{p_z}\right).
\]
Since $\beta = \Prb\left(\alpha_z(X)>a\right)$, the claim now follows.
\end{proof}

We now prove Proposition~\ref{prop:moment}. 

\begin{proof}[Proof of Proposition~\ref{prop:moment}]
Let $X_1,\ldots,X_q\sim P$ be independent test points, independent of the training sample.  Define for $0\le k\le q$, following \eqref{eq:pq},
$
p_k(x_{1:k})=\Prb_S(x_1,\ldots,x_k\in A(S)),
$
with $p_0=1$. We set 
\[
M_k
=
\E_{X_{1:k}}\left[
p_k(X_{1:k})^2
\left(64(d+q)+\log\frac e{p_k(X_{1:k})}\right)^{q-k}
\right],
\]
with the convention that the integrand is zero when $p_k(X_{1:k}) = 0$. It can be seen as a weighted version of the functional \eqref{eq:momentfubini}. An immediate observation is that when $k = q$ by \eqref{eq:momentfubini} we have
\[
\E_{S,T}P(A(S)\cap A(T))^q = M_q.
\]
At the same time, the following recursion holds: for $k=1,\ldots,q,$
\begin{equation}
\label{eq:onestepcontraction}
M_k\le \frac{960}{m}M_{k-1},
\end{equation}
which implies, iterating,
\begin{equation}
\label{eq:momentbound}
\E_{S,T}P(A(S)\cap A(T))^q  = M_q
\le
\left(\frac{960}{m}\right)^q M_0
=
\left(\frac{960}{m}\right)^q
\left(64(d+q)+1\right)^q
\le
\left(62400\,\frac{d+q}{m}\right)^q,
\end{equation}
where we used $64(d+q)+1\le65(d+q)$. Note that \eqref{eq:momentbound} is exactly the claim of the proposition.

It remains only to prove the recursion \eqref{eq:onestepcontraction}. Consider first the deterministic sequence $x_1, \ldots, x_{k} \in \X$ and denote 
\[
\alpha(x)= \alpha_{(x_1, \ldots, x_{k-1})}(x)= \Prb_S(x\in A(S)\mid x_1,\ldots,x_{k-1}\in A(S)).
\]
By the definition of the conditional probability we have
\begin{align}
&p_k(x_{1:k})^2
\left(64(d+q)+\log\frac e{p_k(x_{1:k})}\right)^{q-k} \nonumber
\\
&\qquad= 
p_{k-1}(x_{1:k-1})^2\alpha^2(x_k)
\left(64(d+q)+\log\frac e{p_{k-1}(x_{1:k-1})} + \log\frac{1}{\alpha(x_k)}\right)^{q-k}. \label{eq:twopartsareequal}
\end{align}
Denote for brevity $p_{k-1}= p_{k-1}(x_{1:k-1})$. For $a \in (0, 1]$, consider the real-valued function $\phi: (0, 1] \to \mathbb{R}$ given by $\phi(a) = p_{k-1}^2a^2
\left(64(d+q)+\log\frac e{p_{k-1}} + \log\frac{1}{a}\right)^{q-k}$. A direct differentiation shows that 
\[
0 \le \phi^{\prime}(a) \le 2p_{k-1}^2a
\left(64(d+q)+\log\frac e{p_{k-1}} + \log\frac{1}{a}\right)^{q-k},
\]
and thus, since $\lim\limits_{a \to 0_+}\phi(a) = 0$, for any $y \in (0, 1]$,
\begin{equation}
\label{eq:integral}
\phi(y) \le \int\limits_{0}^y2p_{k-1}^2a
\left(64(d+q)+\log\frac e{p_{k-1}} + \log\frac{1}{a}\right)^{q-k}\,\mathrm{d}a.
\end{equation}
Keeping all variables fixed, we now treat $x_k$ as a random variable $X \sim P$. Integrating the right-hand side of the previous inequality with respect to this variable and using \eqref{eq:integral} and Fubini, we have
\begin{align}
&p_{k-1}^2\E_{X}\alpha^2(X)
\left(64(d+q)+\log\frac e{p_{k-1}} + \log\frac{1}{\alpha(X)}\right)^{q-k} \nonumber
\\
&\le p_{k-1}^2\int\limits_{0}^1 2a\left(64(d+q)+\log\frac e{p_{k-1}} + \log\frac{1}{a}\right)^{q-k}\Prb(a < \alpha(X))\,\mathrm{d}a \nonumber
\\
&\le \frac{320}{m}p_{k-1}^2\int\limits_{0}^1\left(d\log\frac ea+\log\frac e{p_{k-1}}\right)\left(64(d+q)+\log\frac e{p_{k-1}} + \log\frac{1}{a}\right)^{q-k}\,\mathrm{d}a \label{usedlem}
\\
&= \frac{320}{m}p_{k-1}^2\int\limits_{0}^\infty\left(d(1 + u)+\log\frac e{p_{k-1}}\right)\left(64(d+q)+\log\frac e{p_{k-1}} + u\right)^{q-k}\exp(-u)\,\mathrm{d} u, \label{eq:varchange}
\end{align}
where in \eqref{usedlem} we used Lemma \ref{lem:conditioned} and in \eqref{eq:varchange} we used the change of variables $u = \log(1/a)$. Denote $r = 64(d + q) + \log\frac e{p_{k-1}}$. Observe that $d(1 + u)+\log\frac{e}{p_{k-1}} \le r(1+u)$. Moreover, since
$
r\ge 64(q-k),
$
we have
\[
(r+u)^{q-k}
=
r^{q-k}\left(1+\frac ur\right)^{q-k}
\le
r^{q-k}
\exp\left(\frac{(q-k)u}{r}\right)
\le
r^{q-k}\exp(u/64).
\]
Thus, \eqref{eq:varchange} can be bounded by 
\[
\frac{320}{m}p_{k-1}^2\int\limits_{0}^\infty\left(r(1 + u)\right)r^{q-k}\exp(-63u/64)\,\mathrm{d} u \le  \frac{960}{m}p_{k-1}^2r^{q - k + 1},
\]
where we used $\int\limits_{0}^\infty(1+u)\exp(-63u/64) \le 3$. Plugging this into \eqref{eq:twopartsareequal} and integrating with respect to the remaining variables $X_1, \ldots, X_{k-1}$ (using joint independence of $X_1, \ldots, X_{k}$), we get
\begin{align*}
&\E_{X_{1:k}}\left[
p_k(X_{1:k})^2
\left(64(d+q)+\log\frac{e}{p_k(X_{1:k})}\right)^{q-k}
\right] 
\\
&\le \frac{960}{m}\E_{X_{1:k-1}}\left[
p_{k-1}(X_{1:k-1})^2
\left(64(d+q)+\log\frac{e}{p_{k-1}(X_{1:k-1})}\right)^{q-k+1}\right],
\end{align*}
which is the desired recursion \eqref{eq:onestepcontraction}. The claim follows. 
\end{proof}

\appendix

{\footnotesize
\bibliographystyle{alpha}
\bibliography{bibliography}
}

\section{AI disclosure and proof simplification}
\label{app:ai-disclosure}

This appendix describes how LLM assistance was used, after an earlier proof had been obtained, to simplify the argument. All mathematical claims in the final manuscript were checked and written by the authors.

\paragraph{An informal ``one-shot'' AI test.}
 On June 8, 2026, we informally tested public versions of OpenAI GPT-5.5 Pro and Anthropic Claude Opus 4.8 on the problem statement alone, without the hints described below. In these tests, neither model produced a complete proof of Theorem \ref{thm:main}. An example output of GPT-5.5 Pro is available at \url{https://chatgpt.com/share/6a25eaef-7bd4-83ea-a6bb-65e4f0975f69}.

\paragraph{AI assistance disclosure.}
The initial proof of Theorem~\ref{thm:main} was obtained by the first author before the LLM-assisted simplification stage. During this stage, GPT-5.5 Pro was only used to check intermediate calculations and to search for local simplifications. The resulting proof was then verified by the authors, but it was substantially longer than
the one presented here (a sketch of the earlier proof is given in Appendix \ref{app:original-pf}). A subsequent stage then used LLMs to simplify and reorganize this argument: starting from the original proof,
we explored alternative arguments, one of which led to the shorter recursive form given above. 

\paragraph{Proof simplification.}
We include the following LLM-assisted proof simplification process.  Based on the earlier proof, our initial goal was first to compress that proof into a short outline and to use that outline to explore alternative versions of the argument.

First, we prepared a prompt containing the problem statement together with a short list of hints capturing the main ideas.  We then repeatedly removed hints and retested the prompt, testing whether the remaining outline still contained enough information to recover the main steps of the proof in some runs of GPT-5.5 Pro. 

The resulting prompt was then used as a starting point for exploring alternative arguments.  We ran the prompt multiple times and inspected the proposed solutions.  One of these iterations led to the form of the recursive argument presented in this note.  The final proof was then simplified, checked, and written by the authors. 

\paragraph{Distilled prompt.}
The final distilled prompt used in the proof simplification process is reproduced below.  In tests performed on June 8, 2026, this prompt sometimes produced versions of the main recursive construction.  The key hints were: the reduction to bounding moments of two block overlaps; the reduction to the identity \eqref{eq:momentfubini}; and the instruction to study the conditional probabilities
$
\Prb\left(x_k\in A(S)\mid x_1,\dots,x_{k-1}\in A(S)\right).
$

\begin{promptbox}
Let \(\mathcal X\) be an arbitrary domain, and let
\[
\mathcal H \subseteq \{0,1\}^{\mathcal X}
\]
be a binary concept class of VC dimension \(d\ge 1\). Let \(P\) be an arbitrary distribution on \(\mathcal X\), and assume the realizable setting: there exists \(h^\star\in\mathcal H\) such that the labels are generated by \(h^\star\).

Let \(A\) be an arbitrary deterministic consistent selector: whenever \(S\) is a realizable labeled sample, \(A(S)\in\mathcal H\) agrees with all labels in \(S\).

Given \(n\) iid labeled examples, split the sample into three independent blocks \(S_1,S_2,S_3\) of sizes as equal as possible. Let
\[
\widehat h_i=A(S_i),\qquad i=1,2,3,
\]
and define the majority-of-three learner
\[
\widehat h(x)
=
\operatorname{maj}\{\widehat h_1(x),\widehat h_2(x),\widehat h_3(x)\}.
\]

Resolve the following problem completely: prove that there exists a universal constant \(C>0\) such that, for every \(\delta\in(0,1)\),
\[
\Pr\left(
P\{x:\widehat h(x)\neq h^\star(x)\}
\le
C\frac{d+\log(e/\delta)}{n}
\right)
\ge 1-\delta .
\]

The result must hold uniformly over all distributions \(P\), all VC classes \(\mathcal H\) of dimension \(d\), all targets \(h^\star\in\mathcal H\), and all deterministic consistent selectors \(A\). No special ERM rule, compression scheme, margin assumption, disagreement coefficient, distributional regularity, or additional structure may be assumed.

A complete solution must prove the optimal high-probability realizable VC bound
\[
O\left(\frac{d+\log(1/\delta)}{n}\right).
\]

Bounds with extra logarithmic factors, expectation-only guarantees, special cases, or algorithm-specific results do not count unless they are upgraded to the full theorem above.

You may ignore only routine measurability issues.

A minimal hint toward the intended proof:

First reduce to the case \(h^\star\equiv 0\), so every learned hypothesis is an error set avoiding its training block. Then observe that the majority vote errs only where at least two of the three independently learned error sets overlap.

The main step should be a two-block moment bound. For two independent blocks \(S,T\) of size \(m\), prove that for every integer \(q\ge 1\),
\[
\left\|P(A(S)\cap A(T))\right\|_{L_q}
\lesssim
\frac{d+q}{m}.
\]

To prove this, introduce an independent test tuple \(X=(x_1,\dots,x_q)\sim P^q\) and rewrite the overlap moment in terms of
\[
\mathbb E_X \Pr_S(X\subseteq A(S))^2.
\]

For prefixes \(x_1,\dots,x_{k-1}\), study the conditional inclusion probabilities
\[
\Pr\left(x_k\in A(S)\mid x_1,\dots,x_{k-1}\in A(S)\right).
\]
\end{promptbox}

\section{Sketch of the earlier proof of Proposition \ref{prop:moment}}\label{app:original-pf}

For completeness we sketch the original argument, which reaches Proposition~\ref{prop:moment} through an explicit multiscale decomposition rather than the single recursion \eqref{eq:onestepcontraction}; a complete write-up is available at \href{https://divitr.github.io/assets/mot.pdf}{divitr.github.io/assets/mot.pdf}. Write $R_q=\E_{S,T}\,P(A(S)\cap A(T))^q$ and recall from \eqref{eq:momentfubini} the identity $R_q=\E_X\,p_q(X)^2$, together with the chain-rule factorization
\[
p_q(X)=\prod_{\ell=1}^{q}\alpha_{X_{<\ell}}(X_\ell),
\qquad
\alpha_{X_{<\ell}}(x)=\Prb_S\!\left(x\in A(S)\mid X_{<\ell}\subseteq A(S)\right),
\]
where $X_{<\ell}=(X_1,\dots,X_{\ell-1})$. The earlier proof controls $R_q$ by tracking these conditional inclusion probabilities across dyadic scales.

For an anchor tuple $a$ with $p_a=\Prb_S(a\subseteq A(S))>0$ and an integer $k\ge 0$, consider the level set $F_k(a)=\{x:2^{-k-1}<\alpha_a(x)\le 2^{-k}\}$.
The VC dimension enters exactly once, through the relative estimate \eqref{eq:VC-net-explicit}, which gives
\[
P\!\left(F_k(a)\right)\le\frac{C\left(d(k+1)+\log(e/p_a)\right)2^{k}}{m}.
\]
The key feature is that conditioning on the anchor never inflates the VC term, which stays $d(k+1)$; the entire cost of conditioning is carried by the rarity term $\log(e/p_a)$. Each test tuple $X$ is then assigned a dyadic profile $k=(k_1,\dots,k_q)$, with $k_\ell$ the scale of $\alpha_{X_{<\ell}}(X_\ell)$. Writing $\Gamma_k$ for the tuples of profile $k$ and $i=\sum_\ell k_\ell$, the factorization gives $p_q(X)\le 2^{-i}$ on $\Gamma_k$, so by \eqref{eq:momentfubini}, $R_q\le\sum_k 2^{-2i}\,P^q(\Gamma_k)$.

The rarity terms are absorbed by the following observation: on a fixed profile the rarity of each conditioning event is already determined by the earlier levels, since $p_{X_{<\ell}}\ge 2^{-K_\ell-(\ell-1)}$ with $K_\ell=\sum_{j<\ell}k_j$, whence $\log(e/p_{X_{<\ell}})\le C(K_\ell+\ell)$. Substituting this into the level-set bound makes the one step estimate depend only on the profile and not on the particular prefix, so an iterated application of Fubini yields $P^q(\Gamma_k)\le\prod_{\ell=1}^{q}C\left(d(k_\ell+1)+K_\ell+\ell\right)2^{k_\ell}/m$. The whole estimate then collapses to the elementary profile sum
\[
R_q\le\left(\frac{C}{m}\right)^{q}\sum_{k\in\mathbb N_0^{q}}2^{-i}
\prod_{\ell=1}^{q}\left(d(k_\ell+1)+K_\ell+\ell\right)
\le\left(\frac{C(d+q)}{m}\right)^{q}.
\]
This recovers Proposition~\ref{prop:moment} up to the universal constant.

\end{document}